\documentclass{article}

\usepackage[preprint]{neurips_2023}

\usepackage[utf8]{inputenc} 
\usepackage[T1]{fontenc}    
\usepackage{hyperref}       
\usepackage{url}            
\usepackage{booktabs}       
\usepackage{amsfonts}       
\usepackage{nicefrac}       
\usepackage{microtype}      
\usepackage{xcolor}         

\usepackage{drago}

\usepackage{caption}
\usepackage{subcaption}
\usepackage{natbib} 
\usepackage{xcolor}
\definecolor{myothercolor}{HTML}{D7191C}
\definecolor{mycolor}{HTML}{114477}
\usepackage{amsthm}
\newtheorem{theorem}{Theorem}
\newtheorem{corollary}{Corollary}
\newtheorem{lemma}{Lemma}
\newtheorem{example}{Example}
\newtheorem{proposition}{Proposition}
\newtheorem{definition}{Definition}
\newcommand\independent{\protect\mathpalette{\protect\independenT}{\perp}}
\def\independenT#1#2{\mathrel{\rlap{$#1#2$}\mkern2mu{#1#2}}}
\usepackage{booktabs}       
\usepackage{amssymb}
\usepackage{pifont}
\newcommand{\cmark}{\ding{51}}%
\newcommand{\xmark}{\ding{55}}%
\usepackage[off]{editing}
\usepackage{amsmath}

\usepackage{algorithm}
\usepackage{algorithmic}

\title{Fair Clustering: A Causal Perspective}

\author{%
  Fritz Bayer\\
  ETH Zürich\\
  \texttt{frbayer@ethz.ch} \\
  \And
  Drago Ple{\v{c}}ko \\
  Columbia University \\
   \texttt{dp3144@columbia.edu} \\
   \AND
   Niko Beerenwinkel\\
   ETH Zürich \\
   \texttt{niko.beerenwinkel@bsse.ethz.ch} \\
   \And
   Jack Kuipers \\
   ETH Zürich \\
   \texttt{jack.kuipers@bsse.ethz.ch} \\
}

\begin{document}

\maketitle

\begin{abstract}
    Clustering algorithms may unintentionally propagate or intensify existing disparities, leading to unfair representations or biased decision-making. Current fair clustering methods rely on notions of fairness that do not capture any information on the underlying causal mechanisms. We show that optimising for non-causal fairness notions can paradoxically induce direct discriminatory effects from a causal standpoint. We present a clustering approach that incorporates causal fairness metrics to provide a more nuanced approach to fairness in unsupervised learning. Our approach enables the specification of the causal fairness metrics that should be minimised. We demonstrate the efficacy of our methodology using datasets known to harbour unfair biases. 
\end{abstract}


\section{Introduction}

Clustering is a fundamental unsupervised learning technique which is used in various domains to uncover hidden structures and patterns in data \citep{kanungo2002efficient}. However, clustering algorithms may inadvertently propagate or even exacerbate existing disparities related to protected attributes such as gender, race, or religion \citep{chierichetti2017fair, barocas2017fairness}. These disparities can lead to biased representations or decision-making processes that may disproportionately affect marginalized communities (see Example~\ref{example:ex1}). In recent years, there has been a growing interest in developing clustering algorithms that strive to balance the trade-off between quality and fairness of the clustering \citep{chierichetti2017fair, backurs2019scalable, bera2019fair,chhabra2021overview}. 

Despite recent advancements in fair clustering, existing methods predominantly rely on non-causal notions of fairness, neglecting the underlying causal structures that may inadvertently contribute to unfairness within the data. This oversight can be particularly significant, as considering these causal structures might be not only ethically pertinent but also a legal requirement \citep{barocas2016big}. 
Achieving causal fairness involves comprehending and modelling the causal relationships between variables using adequate causal semantics. 
Remarkably, while the causal approach to fairness has been thoroughly explored in the context of supervised learning \citep{kusner2017counterfactual, zhang2018fairness, zhang2018equality, chiappa2019path, wu2019pc, nabi2018fair, plecko2022causal}, its incorporation into unsupervised methods, such as clustering, remains uncharted.

In this work, we introduce a novel causally fair clustering approach that integrates the causal structure of the data. Our approach allows one to specify which causal fairness metrics should be optimised, offering a more targeted perspective on fair clustering. We demonstrate how to perform causally fair clustering to mitigate direct, indirect, and spurious sources of unfairness.

Our key contributions are as follows:
\begin{itemize}
    \item We show that fair clustering algorithms that are based on non-causal notions of fairness can induce direct discriminatory effects from a causal standpoint (see Example~\ref{example:balanced} and Figure~\ref{fig:all}).
    \item We propose an alternative clustering approach that incorporates causal fairness notions, allowing for a more detailed understanding of fairness in unsupervised learning (see Theorem~\ref{theo:alg} and Algorithm~\ref{alg:causally_fair}). \xcomment{Strong Suggestion: Add a Theorem proving the soundness of the Algorithm based on the propositions. You may also call the propositions Lemmas.}
    \item We demonstrate the effectiveness of our causally fair clustering approach on two real-world datasets, showing improvements in causal fairness metrics compared to previous clustering methods (see Figure~\ref{fig:all}).
\end{itemize}

\subsection{Related Work}

Clustering methods with fairness constraints can generally be grouped into those that aim for group fairness and those that aim for individual fairness \citep{chhabra2021overview}.

In the domain of group fairness, \citet{chierichetti2017fair} proposed a method of fair clustering that seeks to create balanced clusters using a technique known as fairlet decomposition. This concept of balanced clusters has been further extended to account for various lp-norm cost functions, multiple groups, relaxed balance requirements, and scalability issues \citep{chierichetti2017fair, rosner2018privacy, bera2019fair, bercea2019cost, ahmadian2019clustering, kleindessner2019guarantees}.

\citet{kleindessner2019fair} introduced a novel approach that focuses on fair representation within selected cluster centres. Their formulation blends k-center objectives with a partition matroid constraint. Building on this, \citet{jones2020fair} put forth an approximation algorithm for this paradigm. The recent advent of the socially fair clustering paradigm, as posited by \citet{abbasi2021fair} and \citet{ghadiri2021socially}, seeks to balance the incurred costs across groups. \citet{makarychev2021approximation} and \citet{goyal2023tight} proposed approximations in this context.


At the intersection of group and individual fairness, recent work has also focused on the fair allocation of public resources in clustering contexts \citep{chen2019proportionally, micha2020proportionally}. These approaches propose that a fair clustering is one where no subset of points has an incentive to assign themselves to a centre outside their designated cluster.

On the other hand, the notion of individual fairness has also received considerable attention. The key is the idea of stability in clusters, where each point in a cluster should not have an average distance to its own cluster larger than to any other cluster \citep{kleindessner2020notion}. \citet{anderson2020distributional} extended the concept of individual fairness to incorporate distributional assignments, ensuring similar points receive analogous assignments, and \citet{brubach2021fairness} introduced pairwise fairness, leveraging point distance to guide separations.

The concept of priority k-center, which deals with usage weights and metric embedding, has also been considered in the context of fair clustering \citep{jung2019center, mahabadi2020individual}. \citet{negahbani2021better} proposed a bi-criteria approximation algorithm for individually fair k-clustering, which has been extended by \citet{vakilian2022improved}.

In conclusion, the landscape of fair clustering is evidently marked by a plethora of non-causal strategies, with various methods prioritising different non-causal notions of fairness and optimising for different types of constraints. In contrast to this prevalent trend, our contribution pivots to a causal perspective on fair clustering, bringing forth a novel approach that optimises causal notions of fairness. 

\begin{example}[Job advertisement]\label{example:ex1}
    Consider the application of clustering in the field of job advertisement, where it is used for target audience segmentation. The clusters are learned based on demographic variables $W$ (e.g., interests, behaviour, income, education, and occupation) and geographic variables $Z$ (e.g., country and ZIP code). The clusters leading to specialised advertisements should be fair with respect to the protected attributes $X$ (e.g., sex, gender or race).  For instance, we want to avoid that ads for high-paying jobs might be unfairly clustered towards a certain gender or ethnic group, perpetuating systemic biases.
\end{example}

\section{Preliminaries}

We use the language of structural causal models (SCMs) as our semantical framework \citep{pearl2000causality}. A SCM is a 4-tuple $\mathcal{M} := \langle V, U, \mathcal{F}, P(u)\rangle$ that encapsulates the following elements:
\begin{itemize}
    \item $V$ and $U$ represent endogenous (observable) and exogenous (unobserved) variable sets, respectively.
    \item $\mathcal{F}$ encompasses a collection of functions $f_{V_i}$ with ${V_i \in V}$. Each individual function is defined as ${V_i \gets f_{V_i}(\pa(V_i), U_{V_i})}$, with $\pa(V_i)\subseteq V$ and $U_{V_i} \subseteq U$.
    \item \(P(u)\) represents a probability distribution over the exogenous variables \(U\).
\end{itemize}

A SCM is associated with a causal diagram, \(\mathcal{G}\), built over the node set \(V\). Within this diagram, a directed edge \(V_i \rightarrow V_j\) emerges if \(V_i\) serves as an argument for \(f_{V_j}\), while an undirected edge \(V_i \bidir V_j\) appears if the related \(U_{V_i}, U_{V_j}\) are interdependent. Throughout, an instantiation of the exogenous variables $U = u$ is termed a \textit{unit}. \(Y_{x}(u)\) captures the potential outcome of \(Y\) upon imposing the condition \(X=x\) for the given unit \(u\). Specifically, \(Y_{x}(u)\) represents the solution for \(Y(u)\) in the submodel \(\mathcal{M}_x\), in which the equations linked to \(X\) within \(\mathcal{F}\) are substituted with the condition \(X = x\).

\subsection{Defining Causal Fairness in Clustering}

To define fairness in unsupervised learning, we adapt the formalism introduced for supervised learning by \citet{plecko2022causal}. Consider the causal diagram displayed in Figure~\ref{fig:sfm}, called the standard fairness model (SFM). In the case of clustering, the outcome variable is the cluster assignment $C$, which can either be a discrete assignment $c_k$ 
to one of the $K$ clusters with $k\in \{1, \dots, K\}$, or a continuous $K$-dimensional cluster membership probability. 

\begin{definition}[Standard Fairness Model, \citealp{plecko2022causal}]
    The SFM is represented by the causal diagram $G_{SFM}$ with endogenous variables $\{X, Z, W, C\}$, which are defined by
    \begin{itemize}
    \item The protected attribute, denoted by $X$ (e.g., gender, race, religion),
    \item The set of confounding variables $Z$, which are not causally influenced by the attribute $X$ (e.g., demographic information),
    \item The set of mediator variables $W$ that could be causally influenced by the attribute $X$ (e.g., educational level or other job-related information),
    \item The cluster assignment, denoted by $C$.
    \end{itemize} 
    Nodes $Z$ and $W$ are possibly multi-dimensional or empty. Furthermore, for a causal diagram $G$, the projection of $G$ onto the SFM is defined as the mapping of the endogenous variables $V$ appearing in $G$ into four groups $X$, $Z$, $W$, $C$, as described above. The projection is denoted by $\Pi_{SFM}(G)$ and is constructed by choosing the protected attribute and grouping the confounders $Z$ and mediators $W$. 
\end{definition}

\begin{figure}[t]
    \centering
    \begin{center}
        \begin{tikzpicture}
    	 [>=stealth, rv/.style={thick}, rvc/.style={triangle, draw, thick, minimum size=7mm}, node distance=18mm]
    	 \pgfsetarrows{latex-latex};
    	 \begin{scope}
    		\node[rv] (0) at (0,1.5) {$Z$};
    	 	\node[rv] (1) at (-2,0) {$X$};
    	 	\node[rv] (2) at (0,-1.5) {$W$};
    	 	\node[rv] (3) at (2,0) {$C$};
    	 	\draw[->] (1) -- (2);
    		\draw[->] (0) -- (3);
    	 	\path[->] (1) edge[bend left = 0] (3);
    		\path[<->] (1) edge[bend left = 30, dashed] (0);
    	 	\draw[->] (2) -- (3);
    		\draw[->] (0) -- (2);
    	 \end{scope}
    	\end{tikzpicture}
    \end{center}
    \caption{Standard fairness model in clustering.}
    \label{fig:sfm}
\end{figure}
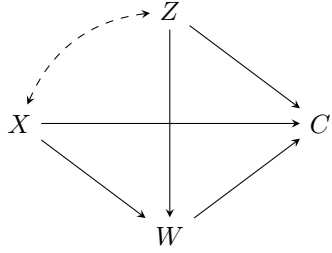

The standard fairness model allows us to define fair cluster associations according to common fairness measures. Note that the variables suggested for each category of the standard fairness model should not be generalised and have to be carefully selected for each dataset. In the clustering setting, the predicted attributes are the cluster associations (or cluster membership probabilities) $C$ of the $K$ clusters. 
For the sake of simplicity, we take $X$ to be binary, while $Z$ and $W$ can be categorical or continuous.

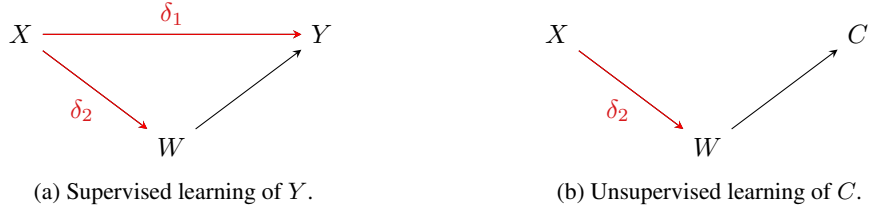
\begin{figure}[t]
     \centering
     \begin{subfigure}[b]{0.49\textwidth}
         \centering
         {
    \begin{tikzpicture}
	[>=stealth, rv/.style={thick}, rvc/.style={triangle, draw, thick, minimum size=7mm}, node  distance=18mm]
	\pgfsetarrows{latex-latex};
	\begin{scope}
		\node[rv] (4) at (2,0) {$Y$};
	 	\node[rv] (1) at (-2,0) {$X$};
	 	\node[rv] (2) at (0,-1.5) {$W$};
        \draw[->] (1) -- (2) node[below, midway, xshift=-5, fill=none, color=myothercolor] {$\delta_2$};
        \draw[->] (1) -- (4) node[above, midway, fill=none, color=myothercolor] {$\delta_1$};
	 	\draw[->] (1) edge[color=red] (2);
		\draw[<-] (4) -- (2);
		\draw[<-] (4) edge[color=red] (1);
	\end{scope}
	   \end{tikzpicture}
        }
         \caption{Supervised learning of $Y$.}
         \label{fig:sup}
     \end{subfigure}
     \hfill
     \begin{subfigure}[b]{0.49\textwidth}
         \centering
         {
         \begin{tikzpicture}
	 [>=stealth, rv/.style={thick}, rvc/.style={triangle, draw, thick, minimum size=7mm}, node distance=18mm]
	 \pgfsetarrows{latex-latex};
	 \begin{scope}
		\node[rv] (5) at (2,0) {$C$};
	 	\node[rv] (1) at (-2,0) {$X$};
	 	\node[rv] (2) at (0,-1.5) {$W$};
		\draw[<-] (5) -- (2);
        \draw[->] (1) -- (2) node[below, midway, xshift=-5, fill=none, color=myothercolor] {$\delta_2$};
	 	\draw[->] (1) edge[color=red]  (2);
	 \end{scope}

	\end{tikzpicture}
    }
         \caption{Unsupervised learning of $C$.}
         \label{fig:unsup}
     \end{subfigure}
        \caption{Neglecting the protected attribute in supervised and unsupervised learning from a causal perspective.}
        \label{fig:sup_unsup}
\end{figure}

A range of possible causal fairness constraints could be added to this model. First, following the approach of \citep{pearl:01}, we use the natural direct effect (NDE) as a notion of a direct effect along the edge $X \to C$ 
\begin{align}
    \text{NDE}_{x_0, x_1}(c_k) = P\big((c_k)_{x_1, W_{x_0}}\big) - P\big((c_k)_{x_0}\big) .
\end{align}
The counterfactual distribution $P((c_k)_{x_0})$ denotes the causal effect of an intervention $do(X = x_0)$ \citep[Ch. 3]{pearl2000causality} on $c_k$, where the counterfactual variable $(c_k)_{x_0}$ denotes the response of $c_k$ to the intervention. Following the same notation, the counterfactual distribution $P\big((c_k)_{x_1, W_{x_0}}\big)$ describes how $c_k$ changes when setting $X$ to $x_1$, but keeping the mediators $W$ at the value it would have taken had $X$ been $x_0$.


To achieve that the constructed cluster assignments $c_k$ are fair with respect to the NDE, one may require that
\begin{align}
    \text{NDE}_{x_0, x_1}(c_k) = \text{NDE}_{x_1, x_0}(c_k) = 0, \label{eq:ndes-zero}
\end{align}
$\forall k \in \{ 1, \dots, K\}$, that is, for every cluster there is a separate constraint. Under the assumptions of the SFM introduced above, the NDE can be identified (i.e., uniquely computed) from the observational data. Thus, the constraint $\text{NDE}_{x_0, x_1}(c_k) = 0$ can be written as
\begin{equation}
    \sum_{z, w} [P(c_k \mid x_1, z, w) - P(c_k \mid x_0, z, w)] \cdot P(w \mid x_0, z)P(z) = 0.
\end{equation}

This constraint ensures that the direct effect, when averaged across different values of $Z, W$, equals 0 and thus guarantees absence of any population direct effect in the clustering assignment. 

Turing around the definition of the NDE, we can obtain a notion for the natural indirect effect (NIE)
\begin{align}
    \text{NIE}_{x_0, x_1}(c_k) = P\big((c_k)_{x_0, W_{x_1}}\big) - P\big((c_k)_{W_{x_0}}\big).
\end{align}
which quantifies the difference of the mediator values $W_{x_0}$ and $W_{x_1}$ on $c_k$ had $X$ been $x_0$. Under the assumptions of the SFM, the NIE is identifiable from observational data, and hence the requirement $\text{NIE}_{x_0, x_1}(c_k)=0$ can be written as
\begin{align}
    \sum_{z, w} P(c_k \mid x_0, z, w) [P(w \mid x_1, z) - P(w \mid x_0, z)]P(z) = 0. 
\end{align}

Further, we define the total variation (TV) as
\begin{equation}
    \text{TV}_{x_0, x_1}(c_k)=P(c_k\mid x_1)- P(c_k\mid x_0),
\end{equation}
which we will relate to the natural direct and indirect effects over the next sections.

\section{Causal Perspective on Fair Clustering}

\subsection{Fairness Through Unawareness in Clustering}


We start with an informal recap of fairness through unawareness (FTU), which means neglecting the protected attribute throughout the analysis in supervised learning before discussing the unsupervised learning case. We exclude $Z$ during this discussion.

\paragraph{Supervised Learning (Classification).} 
Figure~\ref{fig:sup} shows the connections among the variables in a supervised learning scenario. The predicted attribute in this scenario is $Y$. An example could be that we want to predict the salary $Y$ based on the occupation $W$, while accounting for the protected attribute gender $X$. The measured data could be subject to the following unfair biases: first, a direct bias $\delta_1$ along $X \rightarrow Y$, and second, an indirect bias $\delta_2$ along $X \rightarrow W$. Both biases are intrinsic to the data and hence have to be adjusted for in order to be removed. Since the indirect bias $\delta_2$ is propagated via $W \rightarrow Y$, we denote the resulting bias along $X \rightarrow W \rightarrow Y$ as $\hat{\delta}_2$. 


Including the protected attribute $X$ in the classification of $Y$ means that we will learn both the direct and indirect effects in the classifier. Hence, the total variation is a combination of the direct and indirect effect
\begin{equation}
    \text{TV}_{x_0,x_1}(y) = \underbrace{\delta_1}_{X\rightarrow Y} +   \underbrace{\hat{\delta}_2}_{X\rightarrow W \rightarrow Y}.
\end{equation}

Unfortunately, simply ignoring the protected attribute $X$ during the classification process (fairness through unawareness) does not remove the direct or indirect biases. This is due to the correlation between $X$ and $W$, which allows the classifier to gain knowledge about $X$ via $W$. Thus, the direct bias $\delta_1$ may still be propagated via the indirect pathway through the correlation with $W$.  
We denote the bias that is propagated via this correlation $\hat{\delta}_1$, since in practice, this indirect propagation reduces the direct bias $\delta_1$. The total variation in this scenario is
\begin{align}
        \text{TV}_{x_0,x_1}(y) = \underbrace{0}_{X\rightarrow Y} +   \underbrace{\hat{\delta}_1 + \hat{\delta}_2}_{X\rightarrow W \rightarrow Y}.
\end{align}
Hence, even though the protected attribute is not used in the classification, it could be learned indirectly via a correlation with the mediator $W$ and hence will propagate the direct discriminatory bias. 


\paragraph{Unsupervised Learning (Clustering).} 
In supervised learning, all edges ${X\rightarrow W}$, $X{\rightarrow Y}$, and ${W\rightarrow Y}$ reflect the generating process of the data. This is in contrast to the unsupervised learning case (Figure~\ref{fig:unsup}), where only the edge $X \rightarrow W$ depicts the generating process of the data and contains the unfair bias $\delta_2$. Learning the cluster attribute $C$ based on the input $W$ introduces a new assignment mechanism $f_C$ in the SCM that is under our control. Note that if the protected attribute is not used as input to $f_C$, there is no direct edge between $X$ and $C$, which implies the conditional independence ${X \independent C \mid W}$. This conditional independence guarantees the absence of the natural direct effect.

\begin{lemma}[Fairness through unawareness in clustering] \label{prop:prop1}
    If the protected attribute $X$ is not an input of the clustering mechanism $f_C$, the natural direct effect is zero, i.e.,
    \begin{equation}
        \textsc{NDE}_{x_0, x_1}(c_k) = 0. 
    \end{equation}
\end{lemma}

Hence, in contrast to the supervised case, fairness through unawareness allows the removal of the natural direct effect in clustering. However, an unfair bias can still persist via indirect effects. In the next section, we will discuss the implications of fair clustering and show how to remove indirect discriminatory effects. 

\subsection{Fairness Through Balanced Clusters}
Balanced clusters exist when each protected class has equal representation across the clusters \citep{chierichetti2017fair}. This implies that the total variation of a balanced cluster $c_k$ equals $0$, i.e.,
\begin{equation}
    \text{TV}_{x_0, x_1}(c_k)=P(c_k\mid x_1)- P(c_k\mid x_0) = 0 .
\end{equation}
However, previous work has shown that optimizing the TV measure to be zero does not necessarily reduce causal measures \citep{nilforoshan2022causal}.
From a causal perspective, we can further decompose the total variation into the natural direct, natural indirect and experimental spurious effect, for which we reintroduce $Z$ into the discussion.
\begin{proposition}
    [Theorem 4.2 in \citealp{plecko2022causal}]\label{prop:plecko}
    The total variation measure can be decomposed as 
    \begin{align}
        \textup{TV}_{x_0, x_1}(c_k)=& \textup{NDE}_{x_0, x_1}(c_k) - \textup{NIE}_{x_0, x_1}(c_k) +   \textup{Exp-SE}_{x_0, x_1}(c_k),
    \end{align}
    where $\textup{Exp-SE}_{x_0, x_1}(c_k)$ is the experimental spurious effect, defined as 
    \begin{equation}
        \textup{Exp-SE}_{x_0, x_1}(c_k)=\textup{Exp-SE}_{x_1}(c_k) - \textup{Exp-SE}_{x_0}(c_k)
    \end{equation}
    with
    \begin{equation}
        \textup{Exp-SE}_{x}(c_k)=P(c_k \mid x) - P\big((c_k)_{x}\big) .
    \end{equation}
\end{proposition}
The experimental spurious effect measures the disparity in $c_k$ when setting $X=x$ by intervention, compared to observing that $X = x$. 
A direct consequence of Proposition~\ref{prop:plecko} is that if the clusters are balanced, i.e., $\text{TV}_{x_0, x_1}(c_k) = 0$, then
\begin{align}
    \text{NDE}_{x_0, x_1}(c_k) = \text{NIE}_{x_0, x_1}(c_k)-\text{Exp-SE}_{x_0, x_1}(c_k) .
\end{align}
Hence, enforcing the clusters to be balanced can induce a natural direct effect unless $\text{NIE}_{x_0, x_1}(c_k)-\text{Exp-SE}_{x_0, x_1}(c_k)=0$, as illustrated in the following example. 

\begin{example}\label{example:balanced}
    Consider the case where we want to cluster data which includes the country $Z$, race $X$, and the browsing preferences $W$ of internet users. Assume the browsing preferences are identical across $X$, i.e., there is no natural indirect effect $\textup{NIE}_{x_0, x_1}(c_k)=0$. If there is a spurious effect $\textup{Exp-SE}_{x}(c_k)\neq 0$, then enforcing balanced clusters with $\text{TV}_{x_0, x_1}(c_k)=0$ would induce a direct effect from race to cluster in order to counterbalance the spurious effect.
    \xcomment{Not sure this is true. It could also be the case that NIE is just set to 0, and the spurious is set to 0. Thus, this does not seem like a robust statement.}
\end{example}



As a special case, consider a scenario in which we have no spurious effect and neglect the protected attribute in the clustering. This is an interesting scenario, as it bridges the gap between causal and non-causal fairness optimizations. 

\begin{corollary}
    If there is no natural direct effect $\textup{NDE}_{x_0, x_1}(c_k) = 0$ because the protected attribute is neglected, and there is no spurious effect ${\textup{Exp-SE}_{x_0, x_1}(c_k)=0}$, then
\begin{align}
    \textup{TV}_{x_0, x_1}(c_k) = -\textup{NIE}_{x_0, x_1}(c_k) .
\end{align}
\xcomment{The usefulness of this statement is unclear in the big picture. Furthermore, it hinges on a strong assumption of the spurious effect being zero (which may not hold in practice). \\
Instead, there is a result that can strengthen the manuscript by a large margin. It would be great to show that, for a randomly sampled SCM, minimising the TV measure will not guarantee that NIE and Exp-SE are equal to $0$. We have this formal result in the Fair Prediction Theorem in the CFA paper, and I am certain a similar statement holds for your unsupervised setting. I strongly encourage you to have a look at our statement; parse it carefully; and try to derive a statement for your setting from it. }
\end{corollary}

The proof is a direct consequence of Proposition~\ref{prop:plecko}, when $\textup{NDE}_{x_0, x_1}(c_k) = \textup{Exp-SE}_{x_0, x_1}(c_k) = 0$. Hence, in this particular setting, optimizing for balanced clusters is identical to minimising the natural indirect effect. This implies that under the assumption of the SFM, minimising the natural indirect effect could be an alternative to balanced clustering algorithms if the aim is to minimise the total variation. 

\begin{algorithm}[ht]
	\caption{Targeted Learning of Causally Fair Clusters}\label{alg:causally_fair}
    \textbf{Input}: A matrix of variables that we wish to cluster, an SFM mapping $(X,Z,W)$, and a binary vector of length three $(\text{NDE},\text{NIE}, \text{SE})$ specifying which effects should be minimised \\
    \textbf{Output}: Causally fair cluster associations $\phi(X,Z,W)$
    \begin{algorithmic}[1]
	%
    %
    \STATE Optimal transport:
    \IF{$\textsc{SE}=1$}
    \STATE
    Transport $Z \mid x_1$ onto $Z \mid x_0$\\
    Denote the optimal transport map with $\tau^z$    
    \ELSE
    \STATE
    Let $\tau^z$ be the identity map $\tau^z(Z) = Z$
    \ENDIF
    \IF {$\textsc{NIE}=1$}
    \STATE
    Transport $W \mid x_1, \tau^{z}(Z)$ onto $W \mid x_0, Z$ \\
    Denote the transport map with $\tau^w$
    \ELSE
    \STATE
    Transport $W \mid x, \tau^{z}(Z)$ onto $W \mid x, Z$ for ${x \in \{x_0, x_1\}}$ \\
    Denote the transport map with $\tau^w$
    \ENDIF
    \IF{$\textsc{NDE}=1$}
    \STATE
    Fairness through unawareness (neglecting $X$ in input): \\    
    $\phi(X,Z,W)\leftarrow f_C (W,Z)$
    \ELSE
    \STATE
    $\phi(X,Z,W)\leftarrow f_C (X,W,Z)$
    \ENDIF
    \end{algorithmic}
\end{algorithm}

\section{Learning Causally Fair Clusters}\label{sec:cf}

We define causally fair clusters as cluster assignments that minimise the following causal fairness notions: natural direct, natural indirect, and experimental spurious effect. Note that minimisation in this context refers to reducing the absolute value of these effects. In particular, we propose Algorithm~\ref{alg:causally_fair}, which allows to minimise each of the causal fairness notions separately. As part of the input, one needs to specify which of the causal fairness notions should be optimised. 

\subsection{Natural Direct Effect}\label{sec:direct}

By excluding the protected attribute from the clustering process as described in Algorithm~\ref{alg:causally_fair}, we can mitigate the natural direct effect, i.e., $\text{NDE}_{x_0, x_1}(c_k) = 0$, as shown in Lemma~\ref{prop:prop1}. Nevertheless, the protected attribute plays a crucial role in neutralising the natural indirect and experimental spurious effect, as we will show in the subsequent sections.

\subsection{Natural Indirect Effect}\label{sec:indirect}

The natural indirect effect from $X \rightarrow W \rightarrow C$ can be separated in the two pathways $X \rightarrow W$ and $W \rightarrow C$. The pathway $X \rightarrow W$ displays a bias that is intrinsic to the data. In contrast, the pathway $W \rightarrow C$ is learned throughout the clustering mechanism $f_C$ and may propagate an unfair bias.

One option is to use optimal transport \citep{peyre2019computational} to adjust $W$ such that the following condition holds (this approach was used in the context of fairness by \citealp{plevcko2020fair})
\begin{align}\label{eq:ot}
    P(w \mid x_1) - P(w \mid x_0) = 0 , \; \forall \; w \in W.
\end{align}
From Equation~(\ref{eq:ot}), given the SFM, we infer that the natural indirect effect is null. Since, under the assumptions of the SFM, the natural indirect effect is identifiable from observational data, the constraint $\text{NIE}_{x_0, x_1}(c_k) = 0$ can be written as
\begin{align}
    \sum_{z, w} P(c_k \mid x_1, z, w) [P(w \mid x_1, z) - P(w \mid x_0, z)]P(z) = 0. \label{eq:nie-null}
\end{align}

In practice, this optimization may be imperfect and only yield  $\Vert P(W \mid x_0, z) - P(W \mid x_1, z) \Vert_1 \le\delta$, such that the remaining difference can propagate to the natural indirect effect. In this case, we have the following upper bound. 

\begin{lemma}[Natural indirect effect]\label{prop:proposition_nie}
Under the assumptions of the SFM and ${\Vert P(W \mid x_1, z) - P(W \mid x_0, z) \Vert_1 \le\delta_w}$, we have that
    \begin{align}
    \textup{NIE}_{x_0, x_1}(c_k) \leq \Vert\sup_{w} P(c_k \mid x_0, w, Z) \cdot P(Z) \cdot \delta_w \Vert_1
\end{align}
\end{lemma}

\subsection{Experimental Spurious Effect}\label{sec:se}

Analogous to the natural indirect effect, we can also adapt $Z$ such that
\begin{align}\label{eq:ot2}
    P(z \mid x_1) - P(z \mid x_0) = 0 , \; \forall \; z \in Z.
\end{align}

If we neglect the protected attribute in the clustering and adapt both $W$ and $Z$, we can provide the following upper bound. 

\begin{lemma}[Experimental spurious effect]\label{prop:proposition_se}
Assume that we perform fairness through unawareness and adapt $W$ such that $P(w \mid x_1) - P(w \mid x_0) = 0$. For ${\Vert P(Z \mid x_1) - P(Z \mid x_0) \Vert_1 \le\delta_z}$ and under the assumption of the SFM
    \begin{align}
    \textup{Exp-SE}_{x_0, x_1}(c_k) \leq \sup_{z} P(c_k \mid z) \cdot \delta_z .
\end{align}
\end{lemma}

In summary, the natural direct effect can be removed by ignoring the protected attribute in the clustering, and the natural indirect and experimental spurious effect can be removed by adapting $W$ and $Z$, respectively (see Theorem~\ref{theo:alg}). Notably, pre-processing the data is sufficient for obtaining causally fair clusters, and hence, there is no need for more involved in-processing algorithms. 

\begin{theorem}[Soundness of Algorithm 1]\label{theo:alg}
Under the assumption of the SFM, the clustering assignments resulting from Algorithm~\ref{alg:causally_fair} satisfy the following conditions:
\begin{itemize}
    \item \textbf{NDE}: If the protected attribute \( X \) is not part of the clustering mechanism \( f_C \), then in the infinite sample case
    \begin{equation}
        \textsc{NDE}_{x_0, x_1}(c_k) = 0
    \end{equation}
    \item \textbf{NIE}: Algorithm 1 restricts the NIE as  
    \begin{equation}
        \textup{NIE}_{x_0, x_1}(c_k) \leq \Vert\sup_{w} P(c_k \mid x_0, w, Z) \cdot P(Z) \cdot \delta_w \Vert_1 ,
    \end{equation}
    where the variable adaption is subject to the limit ${\Vert P(W \mid x_1, z) - P(W \mid x_0, z) \Vert_1 \le\delta_w}$.
    \item \textbf{Exp-SE}: The Exp-SE is bounded by 
    \begin{equation}
        \textup{Exp-SE}_{x_0, x_1}(c_k) \leq \sup_{z} P(c_k \mid z) \cdot \delta_z ,
    \end{equation}
    where the variable adaption is subject to the limit \({\Vert P(z \mid x_1) - P(z \mid x_0) \Vert_1 \le\delta_z}\).
\end{itemize}
If $\delta_w$ and $\delta_z$ collapse to zero in the optimal transport procedure in the infinite sample case, then ${\textsc{NDE}_{x_0, x_1}(c_k)= \textup{NIE}_{x_0, x_1}(c_k) = \textup{Exp-SE}_{x_0, x_1}(c_k) = 0}$.
\end{theorem}

\section{Experiments}

\begin{figure}[t]
    \centering
    \includegraphics[width=0.96\textwidth]{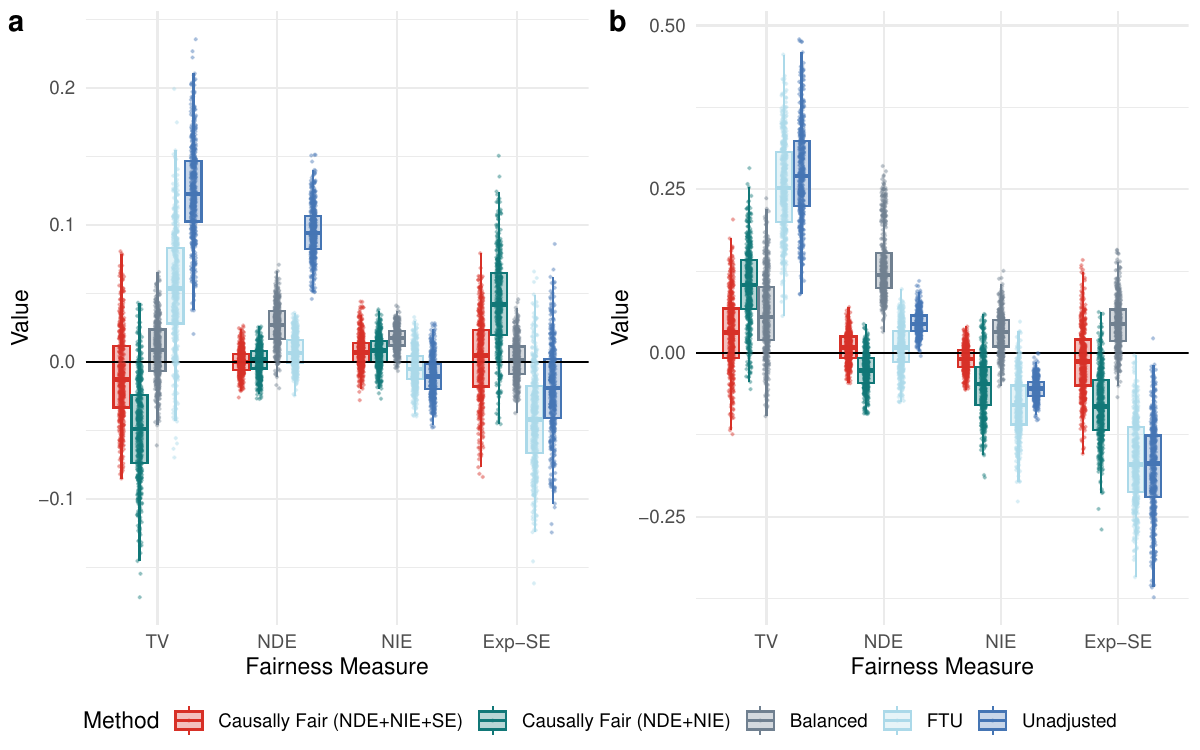}
    \caption{Comparison of fair clustering approaches on the (a) UCI Adult and (b) the COMPAS dataset. Causal Fair: Causally fair clustering, FTU: Fairness through unawareness by neglecting the protected attribute, Unadjusted: Naive clustering without adjustments; TV: total variation, NDE: natural direct effect, NIE: natural indirect effect, and Exp-SE: experimental spurious effect. } 
    \label{fig:all}
\end{figure}

To evaluate the efficiency of our proposed approach, we benchmarked it against several existing clustering methods: balanced clustering, fairness through unawareness clustering, and clustering without any adjustments for fairness, which we refer to as unadjusted clustering. Since Algorithm~\ref{alg:causally_fair} allows one to specify which causal fairness notions are minimised, we implemented two different versions of our approach. One version minimises the NDE, NIE and Exp-SE, which we denote \textit{causally fair (NDE+NIE+SE)}, and the other minimises the NDE and NIE only, which we denote \textit{causally fair (NDE+NIE)}. 

Our experiments were conducted on the Adult dataset from the UCI machine learning repository \citep{lichman2013uci} and the COMPAS dataset \citep{larson2016we}, while the UCI Adult dataset contains demographic data from the 1994 U.S. census and is a common benchmark for income prediction. The COMPAS dataset includes defendant profiles used to assess recidivism risks. Discriminatory biases in the UCI Adult and COMPAS datasets have been well-documented in literature, making them standard benchmarks in both unsupervised and supervised learning \citep{chierichetti2017fair, nabi2018fair, chiappa2019path, larson2016we}. These datasets contain sensitive attributes such as gender and race, which were utilised to measure the fairness of the clustering algorithms. The pre-processed UCI Adult and COMPAS datasets were downloaded from the \textbf{fairadapt} package \citep{plevcko2021fairadapt} and the pre-processing procedure is described in detail in \citep{plevcko2020fair}. The datasets were pre-processed by excluding features such as relationship, final weight, education, capital gain, and capital loss. Furthermore, work class, marital status, and native country were re-categorised into simplified levels, resulting in fewer distinct groups for each variable. 

The SFM model assignments were chosen analogous to the work of \citet{plecko2022causal}, which is based on previous suggestions for the causal structure underlying these datasets \citep{nabi2018fair,chiappa2019path, plevcko2020fair}. Nevertheless, we acknowledge that alternative causal associations could be possible, and inaccuracies in the causal structure could affect the efficiency of our approach. For the Adult dataset within the SFM framework, we chose the confounding variables in $Z$ to include a range of demographic variables, including age, citizenship, and economic region. The Mediator variables in $W$ were chosen to encompass employment and socio-economic factors, including salary, marital status, family size, children, education level, English level, hours worked, weeks worked, occupation, and industry. In the COMPAS dataset within the SFM framework, the confounders in $Z$ include age. The mediator variables in $W$ include various components regarding criminal history and legal circumstances, such as juvenile felony count, juvenile misdemeanour count, other juvenile counts, number of prior offences, degree of charge, and two-year recidivism indicator. As protected attributes, we chose sex and race in the COMPAS and UCI Adult datasets, respectively. We discuss the limitations of our simplification of complex protected attributes in our ethical statement in the Appendix. 

The performance of the algorithms was evaluated based on the following metrics of fairness: total variation (TV), which is a non-causal fairness notion, and the previously introduced causal notions of fairness, including the natural direct effect (NDE), natural indirect effect (NIE), and experimental spurious effect (Exp-SE).

The bias along $X \rightarrow W$ was corrected using optimal transport using the package \textbf{fairadapt}, where quantile regression was executed via random forest quantile regression. Subsequently, the processed data was clustered using the package \textbf{clustMixType} \citep{szepannek2018clustmixtype}. The balanced clustering was performed using the package \textbf{FairMclus} and the fairness measures were calculated using the package \textbf{faircause} \citep{plecko2022causal}. Code implementing our approach and reproducible benchmarks are open-source and publicly available at \url{https://github.com/cbg-ethz/fairClust}. 
All computations were performed in \textbf{R} on a machine with a quad-core Intel core i5 (2.4 GHz) CPU and an Intel iris plus graphics 655 (1536 MB) GPU. 
In the estimation of the causal fairness measures shown in Figure~\ref{fig:all}, we chose 100 inner and five outer bootstrap repetitions. The inner bootstrap repetitions specify the iterations of the fitting procedure, whereas the outer bootstrap repetitions determine the number of bootstrap samples that are taken after the potential outcomes have been obtained from the estimation. This choice represented a trade-off between accuracy and computation time. A higher number of repetitions might increase the accuracy but also result in a longer computation time. We chose a binary clustering, given the high computational cost of the balanced clustering algorithm.

Figure~\ref{fig:all} summarises the benchmark results over both datasets. The unadjusted clustering approach exhibits large discriminatory biases across all fairness measures, including the NDE. In contrast, the fairness through unawareness approach manages to minimise the NDE effectively. Nonetheless, fairness through unawareness displays noticeable effects across the NIE and Exp-SE, which aligns with our expectation as indirect effects are not controlled. The balanced clustering method, despite having a TV near zero, exhibits minor biases in both the NDE and NIE. This illustrates that achieving low TV does not guarantee complete fairness across all measures.

The first version of our causally fair clustering approach, optimising the NDE, NIE and Exp-SE, exhibits minimal biases across all fairness measures. This aligns with our expectation that optimising for causal fairness notions minimises the total variation (see Proposition~\ref{prop:plecko}). In contrast, the second version of our causally fair clustering approach, which only optimises the NDE and NIE, shows minimal biases across the NDE and NIE, but a large Exp-SE. 

These results demonstrate that our novel approach enables us to precisely control which fairness notions are minimised during the clustering process. In practice, this implies we can specify whether confounding variables, such as economic region, can inform the clustering through indirect effects about protected attributes. This may allow for a more nuanced alignment of the clustering process with ethical standards or legal requirements. 

\section{Conclusion}

We introduced a novel causally fair clustering approach as an alternative to existing fair clustering algorithms, which mostly rely on non-causal notions of fairness. Through our experimental analysis of standard datasets, we demonstrated the robustness and efficacy of our approach with respect to causal fairness notions compared to conventional strategies. In particular, our approach allows us to specify which fairness metrics should be optimised, allowing for a more nuanced and targeted optimisation of fairness in unsupervised learning.
In settings where the causal relationships are clearly defined, this may allow for a more thoughtful alignment with ethical principles, legal requirements, and societal needs.
Nevertheless, it is important to recognise that our approach requires a clear understanding of the causal relationships involved. In practice, such relationships may not always be known.

In this work, we optimised for causal fairness notions that investigate both direct and indirect discriminatory effects under the assumptions of the standard fairness model. Our approach lays a practical foundation for causally fair unsupervised learning, which can be extended by exploring other causal fairness notions \citep{barocas2017fairness, castelnovo2022clarification}  in future research.

\bibliography{main}

\clearpage
\newpage
\appendix

\section{Proofs}

\begin{proof}[Proof of Lemma~\ref{prop:prop1}]
    Neglecting the protected attribute in the clustering implies the conditional independence ${X \independent C \mid W}$ and induces the equality 
    \begin{equation*}
        P(c_k \mid w)=P(c_k \mid x_1,w)=P(c_k \mid x_0,w).
    \end{equation*}
    Inserting this equality into the identification expression for the NDE under the assumptions of the SFM proves the proposition
    \begin{align*}
        \textsc{NDE}_{x_0, x_1}(c_k) &= P\big((c_k)_{x_1, W_{x_0}}\big) - P\big((c_k)_{x_0}\big)\\
        &=  \sum_{w} [P(c_k \mid x_1, w) - P(c_k \mid x_0, w)] \cdot P(w \mid x_1)\\
        &= \sum_{w} [P(c_k \mid w) - P(c_k \mid w)] P(w \mid x_1)\\
        &=0 .
    \end{align*}
\end{proof}

\begin{proof}[Proof of Lemma~\ref{prop:proposition_nie}]
Under the condition that ${\Vert P(w \mid x_1, z) - P(w \mid x_0, z) \Vert_1 \le\delta_w}$, we can proof Proposition~\ref{prop:proposition_nie} following Hölder's inequality.
\begin{align*}
    \textup{NIE}_{x_0, x_1}(c_k) & = P\big((c_k)_{x_1, W_{x_0}}\big) - P\big((c_k)_{W_{x_1}}\big)\\
    &= \sum_z P(z) \cdot \Vert P(c_k \mid x_0, z, W) [P(W \mid x_1, z) - P(W \mid x_0, z)]\Vert_1 \\ 
    & \leq \sum_z P(z) \cdot \sup_{w} P(c_k \mid x_0, w) \cdot \Vert P(W \mid x_1, z) - P(W \mid x_0, z) \Vert_1 \\ 
    & \leq \sum_z P(z) \cdot \sup_{w} P(c_k \mid x_0, w) \cdot \delta_w \\
    & = \Vert\sup_{w} P(c_k \mid x_0, w, Z) \cdot P(Z) \cdot \delta_w \Vert_1
\end{align*}
\end{proof}

\begin{proof}[Proof of Lemma~\ref{prop:proposition_se}]

Under the SFM, the experimental spurious effect can be written as
\begin{align*}
    \textup{Exp-SE}_{x_0, x_1}(c_k) = & \textup{Exp-SE}_{x_1}(c_k) - \textup{Exp-SE}_{x_0}(c_k) \\ 
    = & P(c_k \mid x_1) - P\big((c_k)_{x_1}\big)  - P(c_k \mid x_0) - P\big((c_k)_{x_0}\big)\\
    = & \sum_z \big(P(c_k \mid x_1,z) [P(z)-P(z\mid x_1)] \\
    & \quad -P(c_k \mid x_0,z) [P(z)-P(z\mid x_0)]\big) ,
\end{align*}
where we inserted
\begin{align*}
    \text{Exp-SE}_{x}(c_k)&=P(c_k \mid x) - P\big((c_k)_{x}\big) .
\end{align*}

Assuming that we perform fairness through unawareness and adapt $W$ such that ${P(w \mid x_1) - P(w \mid x_0) = 0}$, we can conclude that $P(c_k \mid x, z)=P(c_k \mid z)$. Further, assuming $\Vert P(z \mid x_1) - P(z \mid x_0) \Vert_1 \le\delta_z$, we can proof Proposition~\ref{prop:proposition_se} following Hölder's inequality 
\begin{align*}
    \textup{Exp-SE}_{x_0, x_1}(c_k) = & \sum_z \big(P(c_k \mid x_1,z) [P(z)-P(z\mid x_1)]  -P(c_k \mid x_0,z) [P(z)-P(z\mid x_0)]\big) \\
    = & \sum_z \big(P(c_k \mid z) [P(z)-P(z\mid x_1)] -P(c_k \mid z) [P(z)-P(z\mid x_0)]\big) \\
    = & \sum_z P(c_k \mid z) \big[P(z \mid x_0)-P(z \mid x_1)\big] \\
    = & \Vert P(c_k \mid Z) \big[P(Z \mid x_0)-P(Z \mid x_1)\big] \Vert_1 \\
    \leq & \sup_{z} P(c_k \mid x_0, z)  \cdot \Vert P(Z \mid x_1) - P(Z \mid x_0) \Vert_1 \\ 
    \leq & \sup_{z} P(c_k \mid z) \cdot \delta_z 
\end{align*}

\end{proof}

\begin{proof}[Proof of Theorem~\ref{theo:alg}]
    Assuming the SFM, we proof Theorem~\ref{theo:alg} by applying Lemma~\ref{prop:prop1}, Lemma~\ref{prop:proposition_nie}, and Lemma~\ref{prop:proposition_se}. 
    First, Lemma~\ref{prop:prop1} proofs that the $\textsc{NDE}_{x_0, x_1}(c_k) = 0$ in the infinite sample case. Further, Lemma~\ref{prop:proposition_nie} proofs that the upper bound on the NIE is 
    \begin{equation}
        \textup{NIE}_{x_0, x_1}(c_k) \leq \Vert\sup_{w} P(c_k \mid x_0, w, Z) \cdot P(Z) \cdot \delta_w \Vert_1 ,
    \end{equation}
    where the variable adaption is subject to the limit ${\Vert P(W \mid x_1, z) - P(W \mid x_0, z) \Vert_1 \le\delta_w}$. Finally, Lemma~\ref{prop:proposition_se} proofs the upper bound on the Exp-SE is
    \begin{equation}
        \textup{Exp-SE}_{x_0, x_1}(c_k) \leq \sup_{z} P(c_k \mid z) \cdot \delta_z ,
    \end{equation}
    where the variable adaption is subject to the limit \({\Vert P(z \mid x_1) - P(z \mid x_0) \Vert_1 \le\delta_z}\).
    
    If $\delta_w=\delta_z=0$ in the optimal transport procedure in the infinite sample case, then ${\textsc{NDE}_{x_0, x_1}(c_k)= \textup{NIE}_{x_0, x_1}(c_k) = \textup{Exp-SE}_{x_0, x_1}(c_k) = 0}$.
\end{proof}
 
\section{Ethical Statement}

We would like to point out the following ethical considerations:

\textbf{Simplification of Complex Attributes:} We recognize and affirm the complexity of sex and gender, which can encompass a broad spectrum of identities and expressions. However, due to the constraints of the UCI Adult and COMPAS datasets, we are utilizing a binary representation of sex. This simplification is not intended to overlook or invalidate the diverse realities of sex and gender but is a practical limitation of the analysed data. 

\textbf{Algorithmic Fairness:} Our research aims to enhance fairness in unsupervised learning by addressing both direct and indirect discriminatory effects, but we acknowledge the inherent challenges in capturing all facets of fairness and bias.

\textbf{Causal Relationships:} Our approach requires a clear understanding of the causal relationships involved. In practice, these relationships may not always be known, and erroneous assumptions could lead to unintended consequences.

\textbf{Wider Societal Implications:} We recognize that our work may influence decision-making processes in various domains, such as employment, healthcare, or finance.


\end{document}